\documentclass[runningheads]{llncs}
\usepackage{graphicx}
\usepackage{graphics}
\usepackage{amsmath,amssymb} 
\usepackage{color}
\usepackage{epstopdf}
\usepackage[caption=false]{subfig}
\usepackage{floatrow}
\usepackage{array}
\usepackage{hyperref} 

\newcolumntype{C}[1]{>{\centering\let\newline\\\arraybackslash\hspace{0pt}}m{#1}}
\usepackage[width=122mm,left=12mm,paperwidth=146mm,height=193mm,top=12mm,paperheight=217mm]{geometry}

\newcommand\Mark[1]{\textsuperscript#1}

\begin{document}
\pagestyle{headings}
\mainmatter

\title{Reliable Attribute-Based Object Recognition Using High Predictive Value Classifiers} 

\titlerunning{Reliable Attribute-Based Object Recognition}

\authorrunning{W. Luan, Y. Yang, C. Ferm\"uller, J. S. Baras}

\author{Wentao Luan\Mark{1}, Yezhou Yang\Mark{2}, Cornelia Ferm\"uller\Mark{2}, John S. Baras\Mark{1}
}
\institute{\Mark{1} Institute for Systems Research, University of Maryland, College Park, USA \\
\Mark{2} Computer Vision Lab, University of Maryland, College Park, USA\\
\email{wluan@umd.edu, yzyang@cs.umd.edu, fer@umiacs.umd.edu,  baras@umd.edu}}

\maketitle

\begin{abstract}
We consider the problem of object recognition in 3D using an ensemble of attribute-based classifiers. We propose two new concepts to improve classification in practical situations, and show their implementation in an approach implemented for  recognition  from point-cloud data. First, the viewing conditions can have a strong influence on classification performance. We study the impact of the distance between the camera and the object and propose an approach to fusing multiple attribute classifiers, which incorporates distance into the decision making. Second, lack of representative training samples often makes it difficult to learn the optimal threshold value for best positive and negative detection rate. We address this issue, by setting in our attribute classifiers instead of just one threshold value, two threshold values to distinguish a positive, a  negative and an uncertainty class, and we prove the theoretical correctness of this approach. Empirical studies demonstrate the effectiveness and  feasibility of the proposed concepts.
%
\end{abstract}

\section{Introduction}

Reliable object  recognition from 3D data is a fundamental task for active agents and a prerequisite for many cognitive robotic applications, such as assistive robotics or smart manufacturing. The viewing conditions, such as the distance of the sensor to the  object, the illumination, and the viewing angle, have a strong influence on the accuracy of estimating simple as well as complex features, and thus on the accuracy of the classifiers. A common approach to tackle the problem of robust recognition is to employ attribute based classifiers, and combine the individual attribute estimates by fusing their information  \cite{10VSCH},\cite{Model_blend},\cite{10PFMAOR}.

This work introduces two concepts to robustify the recognition by addressing common issues in the processing of 3D data, namely the problem of classifier dependence on viewing conditions, and the problem of insufficient training data. 

We first study the influence of distance between the camera and the object on the performance of attribute classifiers. Unlike 2D image processing techniques, which usually scale the image to address the impact of distance, depth based object recognition procedures using input from 3D cameras tend to  be affected by distance-dependent noise, and this effect cannot easily be overcome \cite{Noise}.

We propose an approach that addresses effects of distance on object recognition. It considers the response of individual attribute classifiers' depending on distance and incorporates it into the decision making. Though, the main factor studied here is distance, our mathematical approach is general, and can be applied to handle other factors affected by viewing conditions, such as  lighting, viewing angle, motion blur, etc.

To implement the attribute classifiers, usually, the standard threshold method is used to determine the boundary between positive and negative examples.  Using this threshold the existence of binary attributes is determined, which in turn controls the overall attribute space. However, there may not be enough training samples to accurately represent the underlying distributions, which makes it more difficult to learn one good classification threshold that minimizes the number of incorrect predictions (or maximizes the number of correct predictions).

Here we present an alternative approach which applies two thresholds with one aiming for a positive predictive value (PPV), giving high precision for positive classes, and the other aiming for a  negative predictive value (NPV), giving high precision for negative classes.  Each classifier can then have three types of output: ``positive'' when above  the high PPV threshold, ``negative'' when below the high NPV threshold and ``uncertain'' when falling into the interval between the two thresholds. Recognition decisions, when fusing the classifiers, are then made based on the  positive and  negative  results.
More observations thereby are needed for drawing a conclusion, but we consider this trade-off  affordable, since we assume that our active agent can  control the number of observations. 
Note that two threshold approaches have previously been used for the purpose of achieving results of high confidence, for example in  \cite{cost}, and  in probability ratio tests. 

The underlying intuition here is that it should be easier to obtain the high PPV and NPV thresholds than the classical Bayes threshold (minimizing the classification error), when the number of training samples is too small to  represent well the underlying distribution.  Fig.~\ref{threshold} illustrates the intuition. The top figure shows  the ground truth distributions (of the classification score)  of the  positive and negative class. The lower figure  depicts the estimated  distributions from training samples, which are biased  due to an insufficient amount of data. Furthermore, as our experiment revealed, even the ground truth distribution could be dependent on  viewing conditions, which makes it more challenging to learn a single optimal threshold.
In such a case, the system may end up with an inaccurate Bayes threshold. However, it is still possible to select
high  PPV (NPV) thresholds by setting these thresholds (at a safe distance) away  from the  negative (positive) distribution.



\begin{figure}
\centering
\begin{tabular}{cc}
\subfloat[]{
  \includegraphics[width=6.4cm, height=3cm]{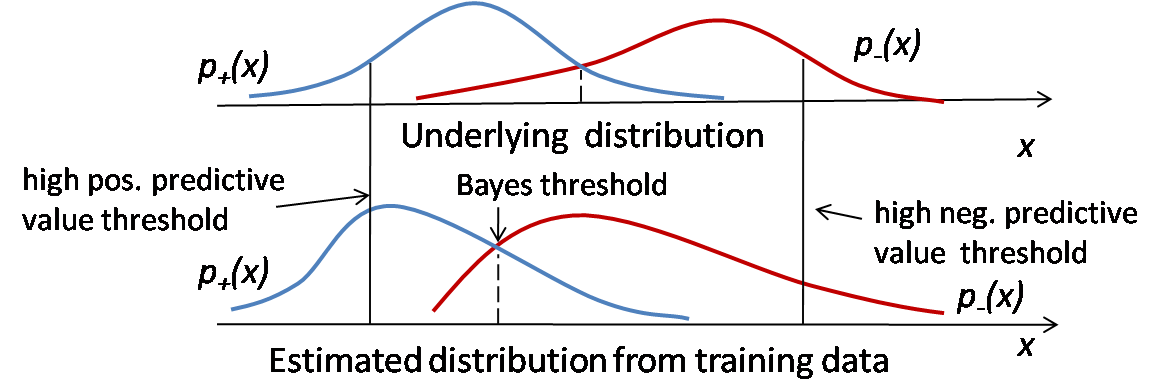}
  \label{threshold}
}
&
\subfloat[]{
   \includegraphics[width=3.5cm, height=3cm]{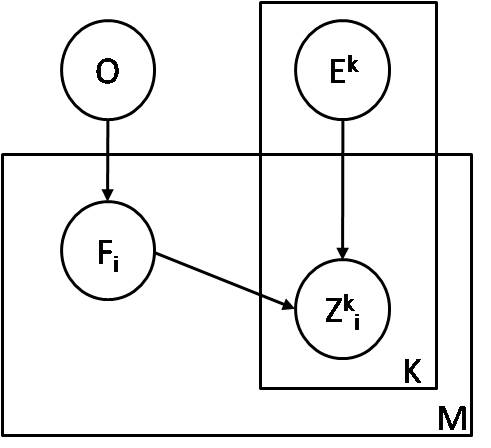}
   \label{relation}
}
\caption{(a): Illustration of common conditional probability density functions of the  positive and negative class. Top: ground truth distribution of the two classes; bottom: a possible distribution represented by the training data. Blue line: positive class; red line: negative class. dashed line: (estimated) Bayes threshold; solid line: high PPV and NPV thresholds. (b): The relationship of Objects ($O$), attributes ($F_i$), environmental variables ($E_k$) and  observations ($Z_{i}^k$) in our model.}
\end{tabular}
\end{figure}

For each basic (attribute) classifier, we can also define a reliable working region 
indicating a fair separation of the distributions of  positive and negative classes.
Hence our approach can actively select ``safe'' samples and discard ``unsafe'' ones in unreliable regions. We prove the asymptotic correctness of this approach in section~\ref{sec:asymp}.

Integrating both concepts, our complete approach to 3D object recognition works as follows: Offline  we learn  attribute classifiers, which are distance dependent. In practice, we discretize the space into $n$ distance intervals, and for each interval we learn classifiers with two thresholds. Also, we decide for each attribute  classifier a  reliable range of distance intervals. During the online process our active system takes RGBD images as it moves around the space. For each input image, it first decides the  distance interval in order to use the classifiers tuned to that interval. Classifier measurements from multiple images are then combined via maximum a posteriori probability (MAP) estimation.

Our work has three main contributions: 1) We put forward a practical framework  for fusing component classifiers' results by taking into account the distance, to accomplish reliable object recognition. 2) We prove our fusion framework's asymptotic correctness under certain assumptions on the attribute classifier and sufficient randomness of the input data. 3) The benefits of introducing simple attributes, which are more robust to viewing conditions, but less discriminative, are demonstrated in the experiment.

\section{Related Work}
Creating  practical object recognition systems that can work reliably under different viewing conditions, including varying distance, viewing angle, illumination and occlusions, is still a challenging problem in Computer Vision. Current single source based recognition methods have robustness to some extent: features like SIFT  \cite{SIFT} or the multifractal spectrum vector (MFS) \cite{MFS} in practice are invariant  to a certain degree to deformations of the scene  and viewpoint changes; geometric-based matching algorithms like BOR3D \cite{BOR3D} and  LINEMOD  \cite{LINEMOD} can recognize objects under large changes in  illumination, where color based algorithms tend to fail. But in complicated working environments, these systems have difficulties to achieve robust performance.

One way to deal with  variations in  viewing conditions is to incorporate different  sources of information (or cues) into the recognition process \cite{trust}. However, how to fuse the information from multiple sources, is still an open problem.

Early fusion methods have tried to build more descriptive features by combining features from sources like texture, color and depth before classification.
For example, Asako et al. builds voxelized shape and color histogram descriptors \cite{10VSCH} and classifies  objects using  SVM, while in \cite{14RT3D} information from color, depth, SIFT and shape distributions is described by histograms and objects are recognized using K-Nearest Neighbors.  

Besides early fusion, late fusion also has gained much attention and achieves good results.
Lutz at al. \cite{10PFMAOR} proposes a probabilistic fusion approach, called MOPED \cite{MOPED},  to combine a 3D model matcher, color histograms and feature based detection algorithm,
where  a quality factor, representing each method's discriminative capability, is integrated  in the final classification score.  
Meta information \cite{9RMAO} can also be added to create a new feature.
Ziang et al. \cite{Model_blend} blends classification scores from SIFT, shape, and color models with meta features providing information about each model's fitness from the input scene, which results in  high precision and recall on the Challenge and Willow datasets.
Considering influences due to viewing conditions, Ahmed \cite{AaptiveBayes} applies an AND/OR graph representation of different  features and updates a Bayes conditional probability table based on  measurements of the environment, such as intensity, distance and occlusions. However, these methods may suffer from inaccurate estimation of the conditional probabilities involved, because of insufficient training data.

In our work, we propose a framework for object recognition using multiple attribute classifiers, which considers both, effects due to viewing conditions and effects due to biased training data that systems face in practice. We implement our approach for an active agent that takes  advantage of multiple inputs at various distances.

\section{Assumptions and Formulation}

Before going into the details and introducing the notation, let us summarize this section. Section~\ref{sec:inference}
defines the data fusion of the different classification results through MAP estimation. Section~\ref{sec:req} proves that MAP estimation will classify correctly under certain requirements and assumptions. The  requirements are  restrictions on the  values of the PPV and NPV. The assumptions are that  our attribute classifiers perform correctly in the following sense: A ground truth positive value should be classified as positive or uncertain and a ground truth negative value should be classified as negative or uncertain. Finally section~\ref{sec:asymp} proves asymptotic correctness of MAP estimation. The estimation will converge, even if the classifiers don't perform correctly, under stronger requirements on the values of the PPV and NPV.

Let the  objects in the database be described by the set $\mathbb{O} = \{o_j\}$ ($j =1, 2, ..., |\mathbb{O}|$). Each object $o_j \in \mathbb{O}$ is  represented by a attribute vector $F^j = [f_{1j}, f_{2j}, ..., f_{Mj}]^T$, where $M$ is the number of attributes.  For the $i$-th attribute $F_i$, there is a corresponding component classifier to identify it. Denote its observation  as $Z_i^k$, where $i$ is  the index for  the  classifier and $k$ is the observation number. Here we consider binary attributes  $f_{ij} \in Range(F_i) = \{0,1\}$, $\forall i \in \{1, 2, ..., M\}$, and there are three possible values for the  observation : $Z^k_i = \{0, 1, u\}~ k \in {1, 2, ,,, K}$, where $u$ represents uncertainty for the case that the classification score falls in the interval between the high PPV and NPV thresholds.

The model also encodes effects due to viewing conditions (or environmental factors). In this work, we study the effect of  distance.
Thus,  $E$ is  the distance between the object and the camera. However, in future work, other environmental factors can be encoded as additional components. Fig.~\ref{relation} illustrates the relationship between objects, attributes, environmental factors and observations in a  graphical model.

In our notation $\mathbb{E}^K = \{E^1, E^2, ..., E^K\}$ represents the environmental variable at each observation, and $\mathbb{Z}_i^K = \{Z_i^1, Z_i^2, ..., Z_i^K\}$ is the set of observation results from the $i$-th classifier. We assume that an  observation of an attribute  $Z_i^k$ only depends on the   ground truth attribute variable $F_i$ and the environmental variable $E^k$. Because we assume that each object $o_j$ can be represented by an $M$-dimension attribute vector $F^j$, we have
$ \label{eq:F}
P(F|O = o_j) =
   \left\{
	\begin{array}{ll}
		1  & \mbox{if } F = F^j, \\
		0 & \mbox{o.w. }
	\end{array}
\right.
$

\subsection{Inference} \label{sec:inference}
 With $K$ observation results 
 $\mathbb{Z}^K = \{\mathbb{Z}_1^K, ..., \mathbb{Z}_M^K\}$ and corresponding environmental conditions  $\mathbb{E}^K$, we want to obtain  the posterior probability of the target object
being object $o_j \in \mathbb{O}$. i.e. $P(O = o_j | \mathbb{Z}^K, \mathbb{E}^K)$. Based on our graphical model we have:
\begin{align}\label{posterior}
\begin{split}
P(O = o_j | \mathbb{Z}^K, \mathbb{E}^K) &  =  \frac{P(O = o_j, \mathbb{Z}^K, \mathbb{E}^K)}{P(\mathbb{Z}^K, \mathbb{E}^K)} 
= \frac{P(O = o_j) P(\mathbb{Z}^K|F = F^j, \mathbb{E}^K)P(\mathbb{E}^K)}{P(\mathbb{Z}^K, \mathbb{E}^K)}  \\
&= \frac{P(\mathbb{E}^K)P(O = o_j)}{P(\mathbb{Z}^K, \mathbb{E}^K)} \prod_{k=1}^K\prod_{i=1}^{M}P(Z_i^k|F_i = f_{ij}, E^k) 
\\& 
=  \lambda P(O = o_j)\prod_{k=1}^K\prod_{i=1}^{M}\frac{P(F_i = f_{ij}|Z_i^k,E^k)}{P(F_i = f_{ij})}
\end{split}
\end{align}
where $\lambda \triangleq \frac{P(\mathbb{E}^K)\prod_{k=1}^K\prod_{i=1}^{M}P(Z_i^k, E^k)}
{P(\mathbb{Z}^K, \mathbb{E}^K)\prod_{k=1}^K\prod_{i=1}^{M}P(E^k)}$. Because
\begin{align}
P(F_i = f_{ij})
                = \sum_t P(O = o_t)P(F_i = f_{ij}|O=o_t) 
                = \sum_{\{t|f_{it} = f_{ij}\}}P(O=o_t)  
\end{align}
Finally, we have
\begin{align} \label{eq:ap}
\begin{split}
& P(O = o_j | \mathbb{Z}^K, \mathbb{E}^K)  = 
 \lambda P(O = o_j) \prod_{k=1}^K\prod_{i=1}^M\frac{P(F_i = f_{ij}|Z_i^k,E^k)}{\sum_{\{t|f_{it} = f_{ij}\}}P(O = o_t)}.
\end{split}
\end{align}

The recognition  $\mathbb{A}$ then is derived using MAP estimation as:
\begin{equation}
\mathbb{A} \triangleq \underset{o_j}{\operatorname{argmax}}~P(O = o_j| \mathbb{Z}^K, \mathbb{E}^K).
\end{equation}
In our framework, we use the high positive and negative predictive value observations ($Z=0,1$) to determine the posterior probability. 

We also take into account the influence of environmental factors. That is,  only  observations from a reliable working  region are adopted in the probability calculation. When the environmental factor is distance, the reliable working region is defined as a range of depth values where the  attribute classifier work reasonably well. We treat a range of distance values as a reliable working region for a classifier, if the detection rate for this range is larger than a certain threshold, and the PPV meets the system requirement.

This requirement for the component classifiers is achievable if the positive conditional probability density function of the classification score has a non-overlapping area with the negative one. Then we can tune the classifier's PPV threshold towards the positive direction (towards left in Fig.~\ref{threshold}) to achieve a high precision with a guarantee of minimum detection rate.


Formally speaking, our $P(F_i = f_{ij}|Z_i^k, E^k)$ is defined as:
\begin{align} \label{eq:CDP}
P(F_i = 1|Z_i^k, E^k) =
   \left\{
	\begin{array}{ll}
		p^+_i   \mbox{~~~~~~~~~~~~~if } e_k \in \mathbb{R}_i ~\&~ z_i^k = 1, \\
		1 - p^-_i  \mbox{~~~~~~~~if } e_k \in \mathbb{R}_i ~\&~ z_i^k = 0, \\
		\sum_{t|f_{it} = f_{ij}}P(O = o_t)  \mbox{~~~~~~~~o.w. }
	\end{array}
\right.
\end{align}
where $\mathbb{R}_i$ is the set of environmental values for which the $i$-th classifier can achieve a PPV $p_i^+$ with a detection rate lower bound. As before, $k$ denotes the $k$-th observation. If the above condition is not met,  either the recognition is done  in an unreliable region or the answer is uncertain. Now equation (\ref{eq:ap}) can be rewritten as:
\begin{align} \label{eq:ap2}
& P(O = o_j | \mathbb{Z}^K, \mathbb{E}^K) =  \lambda P(O = o_j) \prod_{k=1}^K\prod_{i \in \mathbb{I}^k}\frac{P(F_i = f_{ij}|Z_i^k,E^k)}{\sum_{\{t|f_{it} = f_{ij}\}}P(O = o_t)},
\end{align}
where $\mathbb{I}^k = \mathbb{I}^{k+} \cup \mathbb{I}^{k-} $ is the index set of recognized  attributes at the $k$-th observation with  $\mathbb{I}^{k+} = \{i|e^k \in \mathbb{R}_i ~\&~ z_i^k = 1\}$ and $\mathbb{I}^{k-} = \{i|e^k \in \mathbb{R}_i ~\&~ z_i^k = 0\}$.

Intuitively, it means that we only use a component classifier's recognition result when 1) it works in its reliable range; 2) the result satisfies high PPV or NPV thresholds. In Section~\ref{sec:req}, we introduce the predictive value requirements for the component classifiers.

\subsection{System Requirement for the Predictive Value}\label{sec:req}
Here we put forward a predictive value requirement for each component classifier to have correct MAP estimations assuming  there do not exist  false positives or false negatives from observations. 

To simplify our notation, we define the prior probability of object $\pi_j \triangleq P(O = o_j), j = (1,2, ..., N_o)$ and the prior probability of attribute $F_i$ being positive as
$w_i \triangleq \sum_{\{t|f_{it = 1\}}}\pi_t, (i = 1, 2, ..., M)$. For each attribute,
the following ratios  are calculated: $r^+_i \triangleq \max(1,\frac{\max_{\{t|f_{it} = 0 \}} \pi_t}{\min_{\{t|f_{it} = 1\}} \pi_t})$,
$r^-_i \triangleq \max(1,\frac{\max_{\{t|f_{it} = 1 \}} \pi_t}{\min_{\{t|f_{it} = 0\}} \pi_t})$.
$\mathbb{I}^+_{F_j}$ and $\mathbb{I}^-_{F_j}$ are the index sets of positive and negative attributes in $F^j$, and the  reliably recognized attributes' indexes at the $k$-th observation are denoted as 
$\mathbb{I} = \{\mathbb{I}^1, \mathbb{I}^2, ..., \mathbb{I}^K\}$ ($\mathbb{I}^k$ as defined in section~\ref{sec:inference}). We next state the conditions  for correct MAP estimation.

\begin{theorem}
\label{Correct_result}
If the currently recognized attributes $\bigcup_{k} \mathbb{I}^k$ can uniquely identify object $o_j$,
i.e.  $\bigcup_{k} \mathbb{I}^{k+} \subseteq \mathbb{I}_{F^+_j}$, $\bigcup_{k} \mathbb{I}^{k-} \subseteq \mathbb{I}_{F^-_j}$, $\forall t \neq j, \bigcup_{k} \mathbb{I}^{k+} \nsubseteq \mathbb{I}_{F^+_t}$ or
$\bigcup_{k} \mathbb{I}^{k-} \nsubseteq \mathbb{I}_{F^-_t}$, and if  $\forall i \in \{1, 2, ..., M\}$ the classifiers' predictive values satisfy $p^+_i \geq \frac{r^+_i w_i}{1 + (r^+_i - 1)w_i}$ and $p^-_i \geq \frac{r^-_i(1- w_i)}{w_i +r^-_i (1 - w_i)}$, then the 
  MAP estimation result $\mathbb{A} = \{o_j\}$.
\end{theorem}

This requirement  means that if 1) the  attributes can differentiate an object  from others, and 2) the component classifiers' predictive values satisfy the requirement, then for the correct observation input, the system is guaranteed to have a correct recognition result.

\begin{proof}
Based on (\ref{eq:ap2}) and the definition above, the posterior probability of $o_j$ is,
\begin{equation} \label{eq:map_set}
P(O = o_j |\mathbb{Z}^K, \mathbb{E}^K) = \lambda \pi_j \prod_{k=1}^K \bigg( \prod_{i \in \mathbb{I}^{k+}}\frac{p^+_i}{w_i} \prod_{i \in \mathbb{I}^{k-}}\frac{p^-_i}{1 - w_i} \bigg).
\end{equation}

Because the current observed attributes $\bigcup_{k} \mathbb{I}^k$ can uniquely identify $o_j$,  we will have $\forall o_g \in \mathbb{O} / \{o_j\}$, $\exists \mathbb{I}_g \subseteq \bigcup_{k} \mathbb{I}^k $ and $\mathbb{I}_g \neq \emptyset$,
s.t. $\forall i \in \mathbb{I}_g, f_{gi} = 0$ if $i \in \mathbb{I}^{k+}$ or $f_{gi} = 1$ if $i \in \mathbb{I}^{k-}$.
Thus, $\forall o_g \in \mathbb{O} / \{o_j\}$,
\begin{align} \label{eq:complement_set}
\begin{split}
& P(O = o_g |\mathbb{Z}^K, \mathbb{E}^K) = \lambda \pi_g \prod_{k=1}^K
\bigg(\prod_{i \in \mathbb{I}^{k+} /\mathbb{I}_g}\frac{p^+_i}{w_i}
 \prod_{i \in \mathbb{I}^{k+} \bigcap \mathbb{I}_g}\frac{1 - p^+_i}{1 - w_i} \\
&~
 \prod_{i \in \mathbb{I}^{k-} /\mathbb{I}_g}\frac{p^-_i}{1 - w_i}
\prod_{i \in \mathbb{I}^{k-} \bigcap \mathbb{I}_g}\frac{1 - p^-_i}{w_i}\bigg).
\end{split}
\end{align}

Since for each classifier, $p^+_i \geq \frac{r^+_i w_i}{1 + (r^+_i - 1)w_i}$ and
$r^+_i = \max(1, \frac{\max_{\{t|f_{it} = 0 \}} \pi_t}{\min_{\{t|f_{it} = 1\}} \pi_t})$, we have
$\pi_j \frac{p^+_i}{w_i}  \geq \pi_g \frac{1- p^+_i}{1 - w_j}$
and $\frac{p^+_i}{w_i} \geq 1 \geq \frac{1- p^+_i}{1 - w_i}$.
For similar reasons, we have $\pi_j \frac{p^-_i}{1 - w_i}  \geq \pi_g \frac{1- p^-_i}{ w_j}$
and $\frac{p^-_i}{1 - w_i} \geq 1 \geq \frac{1 - p^-_i}{w_i}$.
Also since $\mathbb{I}_g \neq \emptyset$,
we can have (\ref{eq:map_set}) $>$ (\ref{eq:complement_set}), an thus  the conclusion is reached.
\end{proof}

From the proof, we can extend the result to a more general case: if the currently recognized attributes cannot uniquely determine an object,
i.e. there exists a non-empty set $\mathbb{O}' = \{o_j |o_j \in \mathbb{O}, \mathbb{I}_{F^+_j} \supseteq \bigcup_{k}\mathbb{I}^{k+}~\&~ \mathbb{I}_{F^-_j} \supseteq \bigcup_{k}\mathbb{I}^{k-} \}$, the final recognition result
$\mathbb{A} = \underset{o_j \in \mathbb{O}'}{\operatorname{argmax}} ~\pi_j$.
Furthermore, if an equal prior probability is assumed, then $\mathbb{A} = \mathbb{O}'$.

Theorem \ref{Correct_result} proves the system's correctness under correct observations. Next,  for the general case section~\ref{sec:asymp}  proves that MAP estimation asymptotically converges to the actual result under certain assumptions.

\subsection{Asymptotic Correctness of the MAP Estimation}\label{sec:asymp}
Now we are going to prove that MAP estimation will converge to the correct result when 1) the attribute classifiers' PPV and NPV are high enough in their reliable working region, where a lower bound of detection rate exists, and 2) the inputs are sampled randomly.

Denote $d_i$ as the detection rate and $q_i$ as the false-positive rate of the $i$-th attribute classifier when applying the high PPV threshold in its reliable working region. Similarly, for the high NPV threshold, $s_i$ denotes  the true negative rate and  $v_i$ denotes the false negative rate.

\begin{theorem}
\label{asymptotic_correct}
 We assume that the inputs are sampled sufficiently random such  that each attribute classifier gets the same chance to work in its reliable region where a lower bound exists for its detection rate, $0 < A < d_i  \leq 1$ and all the objects have different positive attributes,
 i.e. $\forall i,j, ~i \neq j ~ s.t.~ \mathbb{I}_{F^+_i} \nsubseteq \mathbb{I}_{F^+_j}$. If the component classifiers' predictive values $p^+_i$ and $p^-_i$ are high enough,
 MAP estimation will converge to the correct result asymptotically with an increasing number of observations.
\end{theorem}
\begin{proof}
Consider the worst case, where only two candidates  $\mathbb{O} = \{o_1, o_2\}$ exist. Without loss of generality, assume $o_1$ has positive attributes $\mathbb{I}_{F^+_1} = \{1, 2, ... ,M_1\}$ and $o_2$ has all the remaining positive attribute $\mathbb{I}_{F^+_2} = \{M_{1+1},M_{1+2}, ..., M\}$, where
$M_1 \geq 1$. Also assume $o_1$ is the ground truth object.
In this case all the false-positive and false-negatives  will drive the estimation  toward $o_2$.

Based on (\ref{eq:ap2}), the posterior probability distributions of $o_1$ and $o_2$ can be written as:
\begin{equation}\label{eq:o1}
P(O = o_1 |\mathbb{Z}^K, \mathbb{E}^K) = \lambda \pi_1 \prod_{i=1}^{M_1}(\frac{p^+_i}{w_i})^{n^+_i}
(\frac{1 - p^-_i}{w_i})^{n^-_i}
 \prod_{i = M_1 + 1}^M(\frac{1 - p^+_i}{1 - w_i})^{n^+_i}
(\frac{p^-_i}{1 - w_i})^{n^-_i}
\end{equation}
\begin{equation}\label{eq:o2}
P(O = o_2 |\mathbb{Z}^K, \mathbb{E}^K) = \lambda \pi_2 \prod_{i=1}^{M_1} (\frac{1 - p^+_i}{1 - w_i})^{n^+_i}
(\frac{p^-_i}{1-w_i})^{n^-_i}
 \prod_{i = M_1+1}^M(\frac{p^+_i}{w_i})^{n^+_i}
 (\frac{1-p_i^-}{w_i})^{n_i^-}, 
 \end{equation}
where $n^+_i$ and $n^-_i$ are the number of positive and negative recognition results of the  $i$-th attribute. Denote $n$ as the number of times the  $i$-th classifier works in its reliable region $\mathbb{E}^i$. Based on the central limit theorem, we have $P(n^+_i > n \frac{d_i}{\alpha}) = 1$ and $P(n^-_i < n\alpha v_i) = 1$ for $i = 1,2,...,M_1$ when $n$ goes to infinity and $\alpha$ can be any positive constant larger than $1$.

For the same reason, we have $P(n^+_i < n \alpha q_i) = 1$ for $i = M_1+1, ...,M$ when $n$ goes to infinity. We use the same $n$ here assuming same likelihood of  reliable working regions for each classifier. Actually it does not matter if there is a constant positive factor on $n$, which means that the chances for the classifiers' reliably working region may be  proportional.

Dividing (\ref{eq:o1}) by  (\ref{eq:o2}), we obtain:
\begin{align}\label{compare}
\begin{split}
& \frac{P(O = o_1 |\mathbb{Z}^K, \mathbb{E}^K)}{P(O = o_2 |\mathbb{Z}^K, \mathbb{E}^K)} = 
\frac{\pi_1}{\pi_2}
\frac{ \prod_{i=1}^{M_1}(\frac{p^+_i / w_i}{(1 - p^+_i) / (1 - w_i)})^{n^+_i}
(\frac{(1-p^-_i) / (w_i)}{p^-_i / (1-w_i)})^{n_i^-}
}
{\prod_{i=M_1+1}^M  (\frac{p^+_i / w_i}{(1 - p^+i) / (1 - w_i)})^{n^+_i}
(\frac{(1-p^-_i) / w_i}{p_i^- / (1-w_i)})^{n_i^-}
}
\\
& \geq \frac{\pi_1}{\pi_2} \frac{ \prod_{i=1}^{M_1}(\frac{p^+_i / w_i}{(1 - p^+_i) / (1 - w_i)})^{n \frac{d_i}{\alpha}}
(\frac{(1-p^-_i) / (w_i)}{p^-_i / (1-w_i)})^{n\alpha v_i}
}
{\prod_{i=M_1 + 1}^M  (\frac{p^+_i / w_i}{(1 - p^+i) / (1 - w_i)})^{n\alpha q_i}} \\
& \text{~~~~~($p^+_i$, $p_i^-$ larger than the threshold in theorem \ref{Correct_result})}
\\
& = c_1 \bigg(c_2  \frac{\prod_{i=1}^{M_1}(\frac{p^+_i}{1 - p^+_i})^{\frac{d_1}{\alpha}}
(\frac{1-p^-_i}{p_i^-})^{\alpha v_i}
}
{\prod_{i = M_1+1}^M (\frac{p^+_i}{1-p^+_i})^{\alpha q_i} }\bigg)^n
\geq c_1 \bigg(c_2 \frac{\prod_{i=1}^{M_1} (\frac{p^+_i}{1 - p^+_i})^{\frac{A}{\alpha}}
(\frac{1-p^-_i}{p_i^-})^{\alpha \frac{1-p^-_i}{w_i} }
}
{ \prod_{i = M_1+1}^M (\frac{p^+_i}{1-p^+_i})^{\alpha \frac{(1-p^+_i)}{1-w_i} } }  \bigg)^n \\
& \text{~~~~(for the upper bound of $q_i$ and $v_i$ see (Eq.~\ref{eq:up_qv}))} 
\end{split} 
\end{align}
Because $\underset{p \to 1}{\lim}\frac{p}{1 - p} = \infty$ and $\underset{p\to 1}{\lim}(\frac{p}{1 - p})^{1-p} = 1$, the division will be larger than $1$ when the predictive value of each classifier is high enough, which means the MAP estimation will yield $o_1$ asymptotically.

The proof of upper bound of $q_i$ and $v_i$:
\begin{align} \label{eq:up_qv}
q_i = P(Z_i = 1 | F_i = 0) = \frac{P(Z_i = 1)(1-p^+_i)}{1 - w_i} \leq \frac{1-p^+_i}{1 - w_i}
\\
v_i = P(Z_i = 0 | F_i = 1) = \frac{P(Z_i = 0)(1-p^-_i)}{ w_i} \leq \frac{1-p^-_i}{ w_i}
\end{align}
\end{proof}

Beyond providing theoretical background, in the next section we perform  experiments on a real object recognition task to first demonstrate the influence of the environment, and then  to validate our framework's performance.

\section{Experiments}
In this section, we demonstrate our framework on the   task of recognizing objects on a table top. 
We first build a pipeline to collect our own data\footnote{The dataset is available from \href{http://ece.umd.edu/~wluan/ECCV2016.html}{\url{http://ece.umd.edu/~wluan/ECCV2016.html}}
}. The reason for collecting our own data is that other available RGBD datasets \cite{RGBD},\cite{BigBIRD} 
focus on different aspects, usually pose or multiview recognition, and do not contain a sufficient amount of samples from varying  observation distances.

Three experiments are conducted to show 1) the necessity of incorporating environmental factors (the recognition distance in our case) for object recognition; 2) the performance of the high predictive value threshold classifier in comparison to the  single threshold one;  and 3) the benefits of incorporating less discriminative attributes for extending the working range of classifiers.

\subsection{Experimental Settings}
The preprocessing pipeline is illustrated in Fig.~\ref{preprocess}. After a point cloud is grabbed from a 3D camera such as Kinect or Xtion PRO LIVE, we first apply a passthrough filter to remove points that are too close or too far away from the camera. Then the table surface is located by matching the point could to a 3D plane model using  random sample consensus (RANSAC), and only points above the table are kept. Finally, on the remaining points, Euclidean clustering is employed to generate object candidates, and  point clouds with  less than $600$ points are discarded.
\begin{figure}
    \centering
        \subfloat[]{\includegraphics[scale = 0.24]{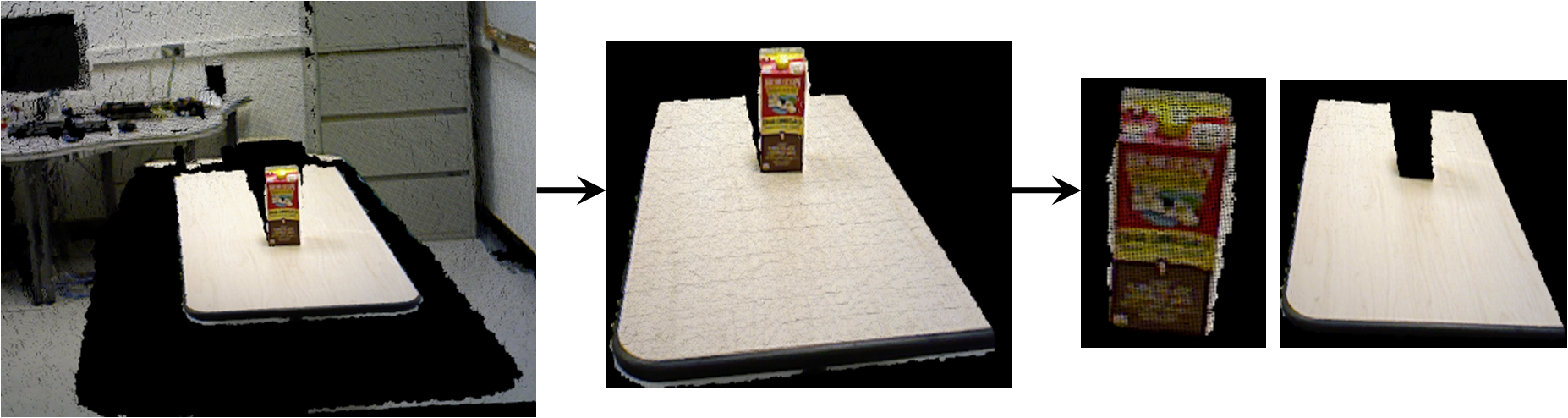} \label{preprocess}} ~
        \subfloat[]{\includegraphics[width = 4.0cm, height=2.4cm]{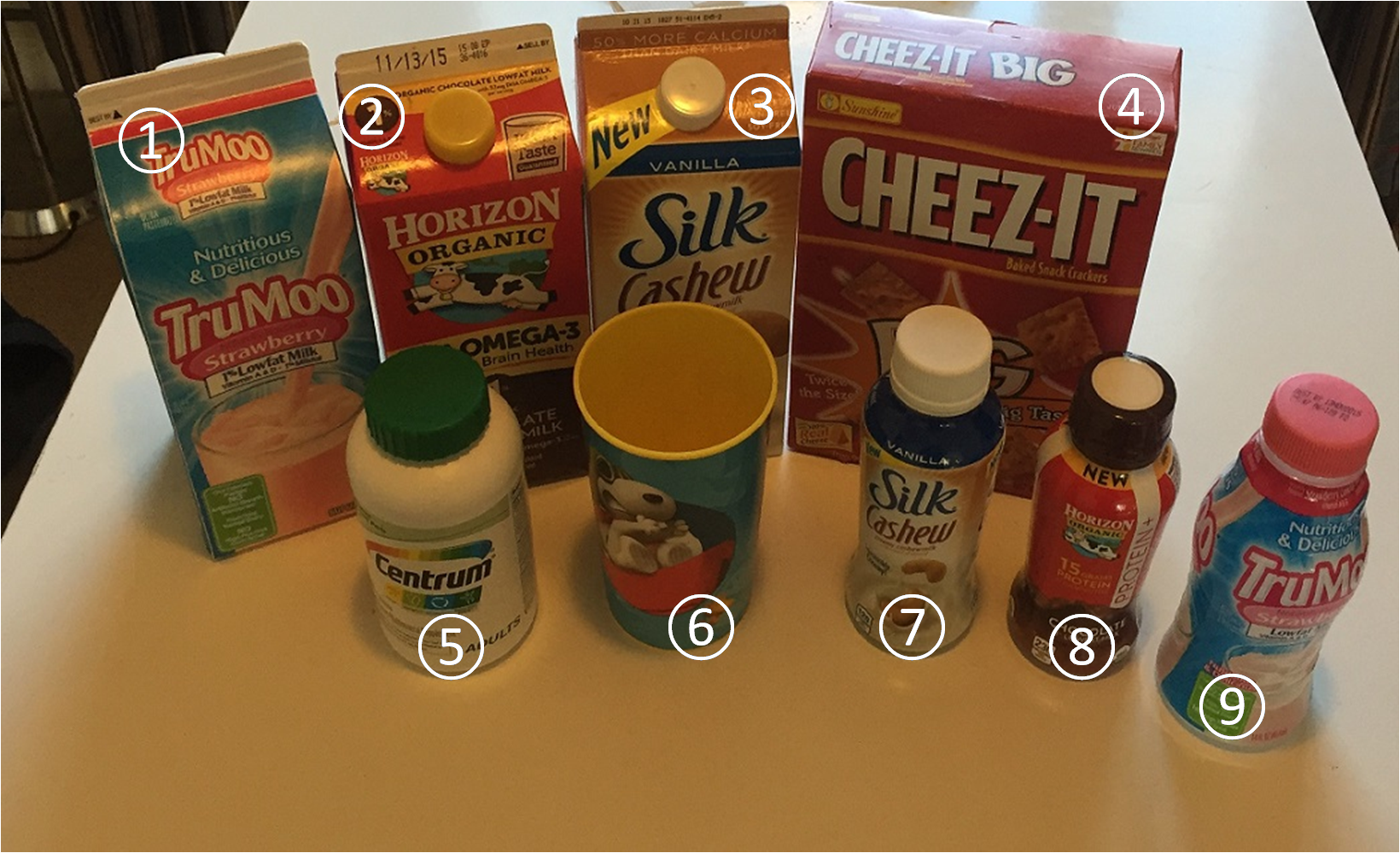}  \label{objects}}
        \caption{(a) Illustration of the preprocessing pipeline. Left: input; Middle: point cloud after passthrough filtering; Right: segmented candidate object and removed table surface. (b) The objects we use in the task and their IDs. } 
\end{figure}

For the segmented point clouds, three categories of classifiers are applied, which are tuned to  attributes of fine shape, coarse shape, and color.


Fine shape is recognized by the Viewpoint Feature Histogram (VFH) descriptor, which encodes a point cloud into a $308$ dimensional vector. Radu \cite{VFH} provides a pipeline of computing VFH features and retrieving the minimum feature distance matching  by fast approximate K-Nearest Neighbors, implemented in the 
Fast Library for Approximate Nearest Neighbors (FLANN) \cite{flann}. However, this approach tends to generate false positives when matching different point clouds 
with very different  distances to the camera. Thus, we adapt the original recognition pipeline to a two step matching. We first pick up model point clouds from our database with  similar distance to test input point cloud. Among the nearby template point clouds, we use the minimum VFH feature matching distance as the classification score. Both steps use FLANN to accelerate neighbor retrieval, where  the former step uses the Euclidean distance and the latter  the Chi-Square distance.

As another type of attribute, we use coarse shape, which is less selective than the fine shape attribute. Our experiments later on demonstrate  its advantage of having a larger working region, thence it can help increase the system's recognition accuracy over a broader range of distance. Two coarse shapes,   cylinders and planar surfaces, are recognized by fitting a cylindrical and a plane model, whose coefficients are estimated by RANSAC. The percentage of outlying points is counted as the classification score for  the  shape. Thus, a  lower score indicates better coarse attribute fitting in our experiment.

The last type of attribute in our system is color, which is used to augment the system's recognition capability. To control the influence of illumination, all samples are collected under one stable lighting condition.
The color histogram is calculated on point clouds after Euclidean clustering, where few background or irrelevant pixels are involved. The Hue and Saturation channels of color are discretized into $30$ bins $(5 \times 6)$, which works well for  differentiating the major colors.

As shown in Fig.~\ref{objects}, there are  $9$ candidate objects in our dataset. To recognize them, we use $5$ fine shape attributes: shape of cup, bottle, gable top carton, wide mouse bottle, and box; $2$ coarse shape attributes: cylinder and plane surface; $3$ major colors: red, blue and yellow. The attributes for all objects are  listed in Table~\ref{tbl:attr}. In the following experiments, we fix the pose of objects, and set the recognition distance as the only changing factor.

\begin{table}
\scriptsize
\begin{center}
 \begin{tabular}{|C{1cm}||C{1cm}|C{1cm}|C{1cm}|C{1cm}|C{1cm}|C{1cm}|C{1cm}|C{1cm}|C{1cm}|C{1cm}|}
 \hline
 Object ID&plane surface&cylinder&gable top carton shape&box shape&wide mouth bottle shape&cup shape&bottle shape&red color&blue color&yellow color\\ 
 \hline
$1$ &  \checkmark&-&\checkmark&-&-&-&-&-&\checkmark&-\\
 \hline
$2$ &  \checkmark&-&\checkmark&-&-&-&-&\checkmark&-&-\\
 \hline
$3$ &  \checkmark&-&\checkmark&-&-&-&-&-&-&\checkmark\\
 \hline
$4$ &\checkmark&-&-&\checkmark&-&-&-&\checkmark&-&-\\
 \hline
$5$ &-&\checkmark&-&-&\checkmark&-&-&-&-&-\\
 \hline
$6$ &-&\checkmark&-&-&-&\checkmark&-&-&\checkmark&-\\
 \hline
$7$ &-&\checkmark&-&-&-&-&\checkmark&-&-&\checkmark\\
 \hline
$8$ &-&\checkmark&-&-&-&-&\checkmark&\checkmark&-&-\\
 \hline
$9$&-&\checkmark&-&-&-&-&\checkmark&-&\checkmark&-\\
\hline
\end{tabular}
\caption{Object IDs and their list of attributes}
 \label{tbl:attr}
\end{center}
\end{table}

\subsection{Experimental Results}

{\bf EXPERIMENT ONE: }The first experiment is designed to validate our claim that the classifiers' response score distributions are  indeed distance variant. Therefore, it is necessary to integrate distance in a robust recognition  system.

Taking the fine shape classifier of bottle shapes as  example, we divide the distance range between $60$ cm and  $140$ cm into $4$ equally separated  intervals and collect positive samples (object id $7,8,9$) and negative samples from the remaining $9$ objects in each distance interval. The number of positive samples in each interval is $120$ with $40$ objects from each positive instance, while the number of negative samples is $210$ with $35$ from each instance. The distribution of the bottle classifier's response score is approximated by Gaussian kernel density estimation with a standard deviation of $3$, and plotted in Fig.~\ref{fig:distribution}.
\begin{figure}
\centering
\begin{tabular}{cc}
\subfloat[]{\includegraphics[width=5cm, height = 3.1cm]{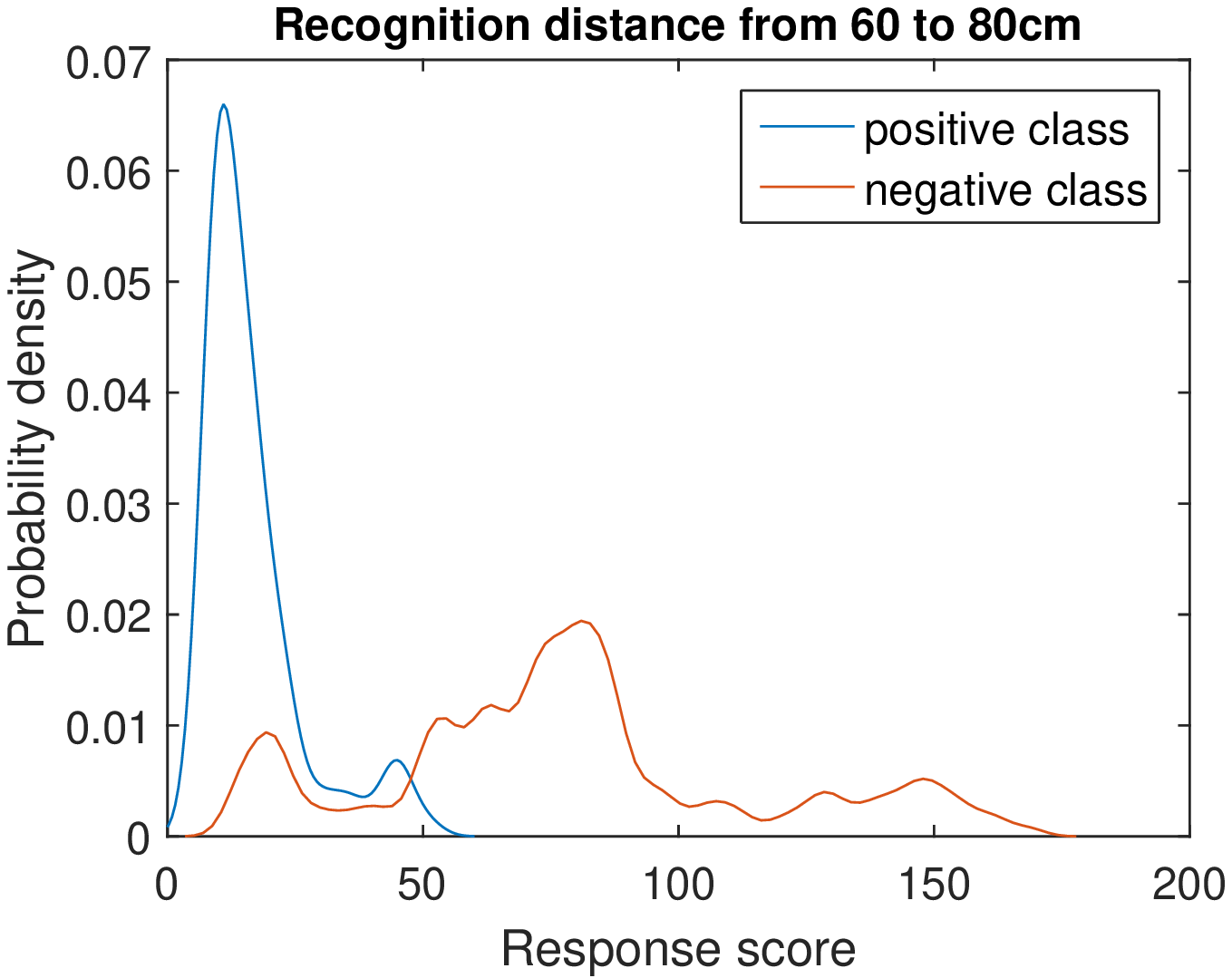}}
   & \subfloat[]{\includegraphics[width=5cm, height = 3.1cm]{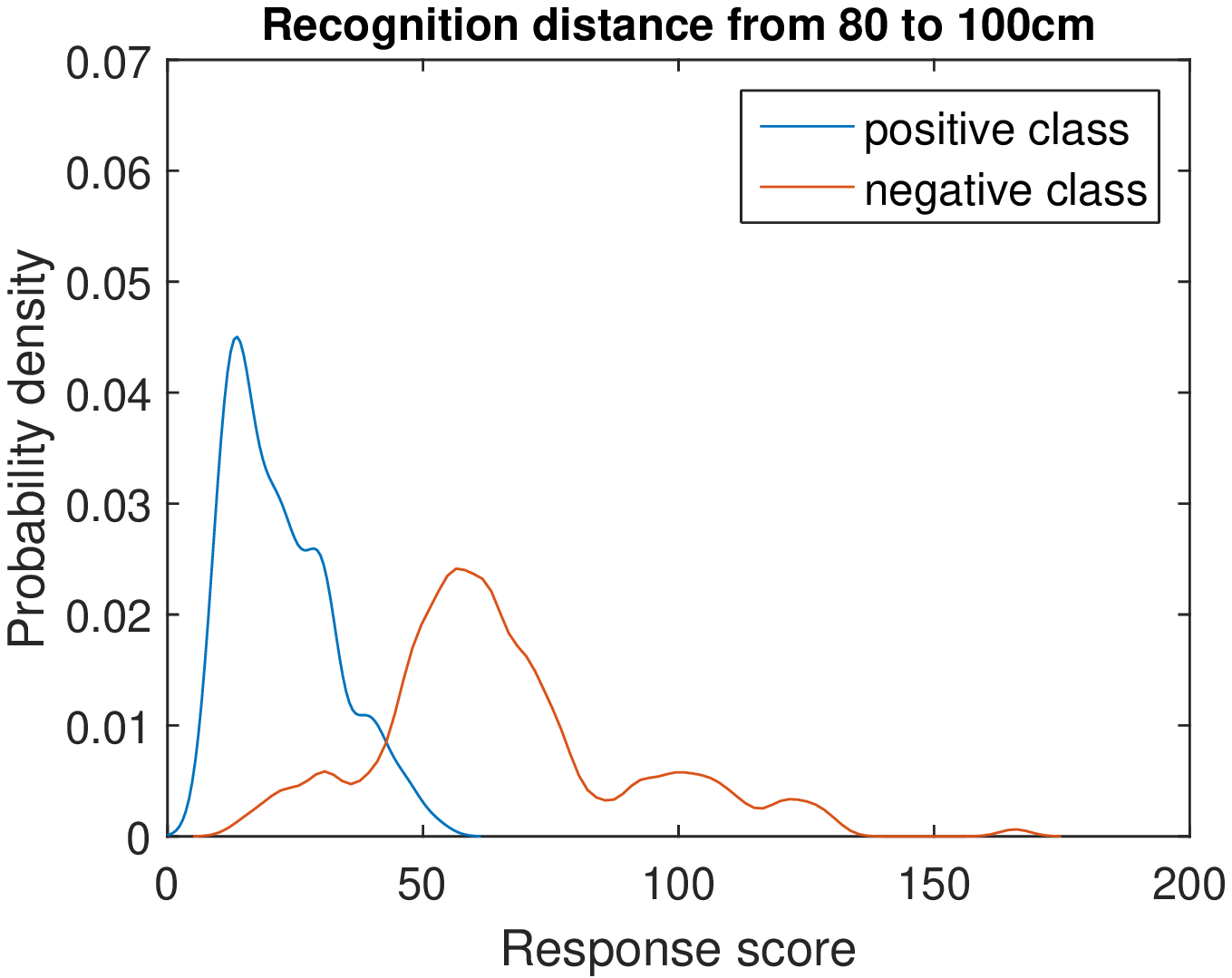}}\\
\subfloat[]{\includegraphics[width=5cm, height = 3.1cm]{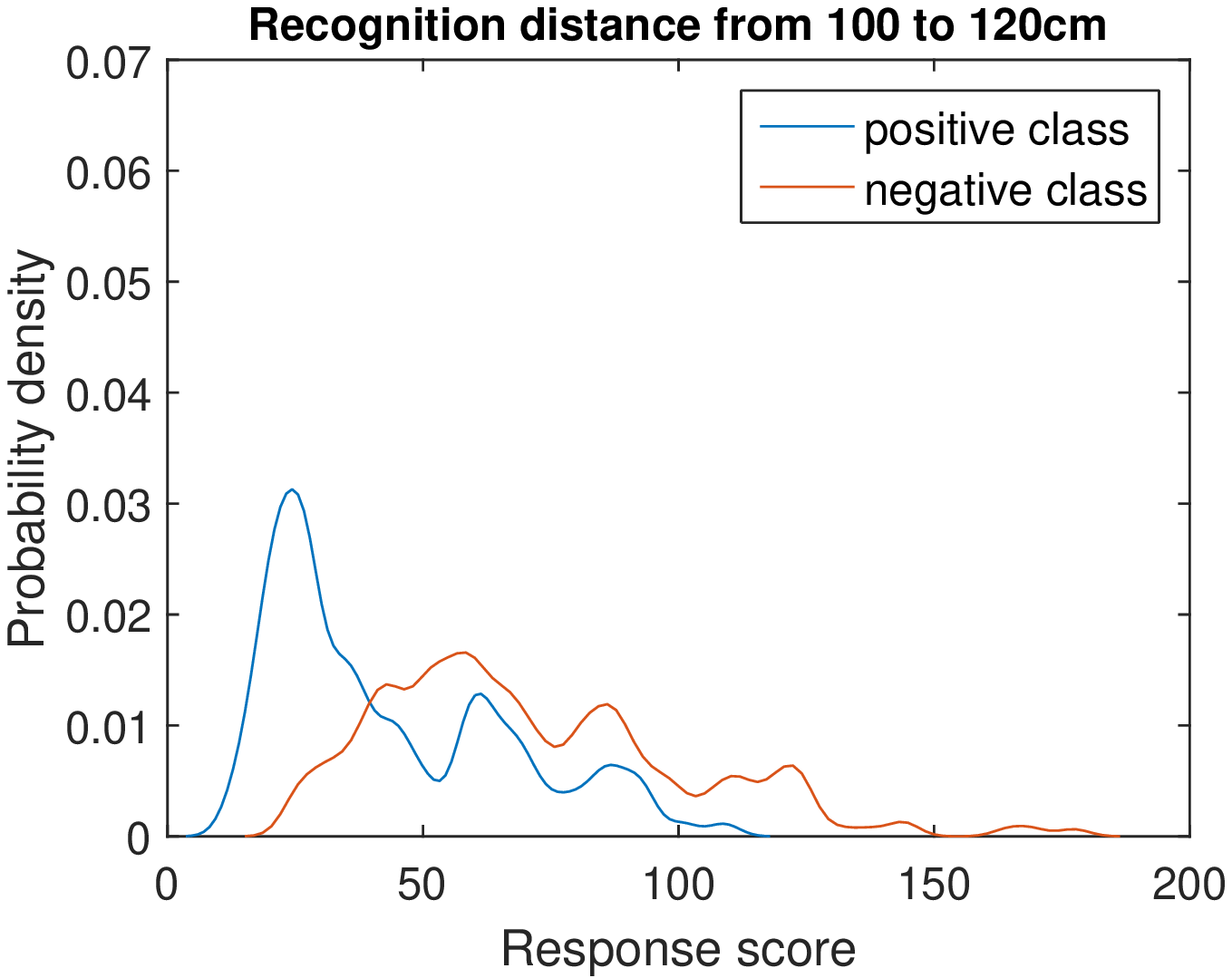}}
   & \subfloat[]{\includegraphics[width=5cm, height = 3.1cm]{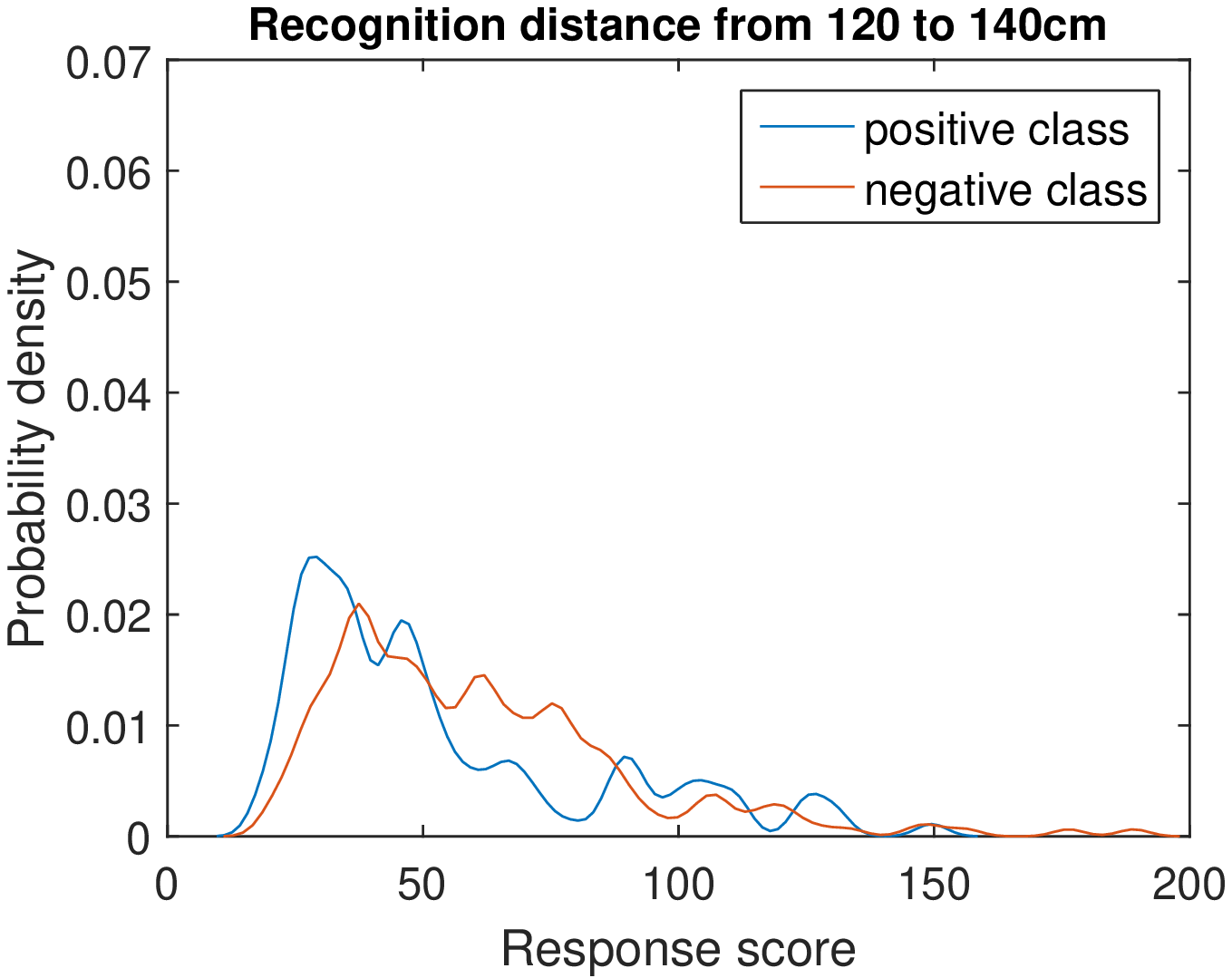}}\\
\end{tabular}
\caption{Estimated distribution of bottle shape classifier's response score  under $4$ recognition distance intervals.}\label{fig:distribution}
\end{figure}

We observe that the output score distribution depends on the recognition distance interval. Therefore, relying on one single classification threshold across all the distance intervals would introduce additional error. More importantly, we observe that with a larger distance, the  area of overlap between   the positive and negative distributions, becomes wider, which makes classification more difficult.

{\bf EXPERIMENT TWO: }Experiment one demonstrated the difficulty of learning a distance-variant ground truth distribution and  corresponding classification thresholds. Therefore, we propose to use two high predicative value thresholds when multiple inputs are available.
The second experiment is designed to validate this idea by comparing the classification accuracy of an estimator that 1) uses  two high predicative value thresholds, with an estimator that uses 2) one optimal Bayes threshold minimizing the error in the training data.

To have  a fair comparison, we set our task as recognizing $5$ objects (id $1,4,5,6,9$ ) with $5$ fine shape attributes such that each object contains one positive attribute that uniquely identifies it. 
Both training and testing point clouds are collected at a distance of $100$ cm to $120$ cm . 
To learn the classification threshold, we sample $26$ point clouds for each object and uniformly select $20$  for  training. The testing data for each object consists of $22$ point clouds that we can randomly choose from to simulate the scenario of an active moving observer  gathering multiple  inputs. 
Here we want to mention a special case. When our framework is uncertain based on the current input, it randomly select (with equal probability) one of the possible objects. The classification accuracy between using a single threshold and using two high predicative value thresholds are shown in Fig.~\ref{fig:error} respectively.



We can see that both methods' error rates decrease when  the number of observations increases. The approach using two thresholds (the red line) has lower error rate than the one using a single threshold (the blue line). The green line shows the error introduced by random selection, when our framework cannot make a sole decision. The major part of the error in the two thresholds method is due to this error. It is worth mentioning that under theoretical conditions, the classical Bayes single threshold should still be the best in minimizing the classification error. Our method provides an alternative for cases when the training data in real world scenarios does not represent  well the underlying distribution. 
\begin{figure}
\centering
\begin{tabular}{cc}
\subfloat[]{
 \includegraphics[width=5.3cm, height=3.3cm]{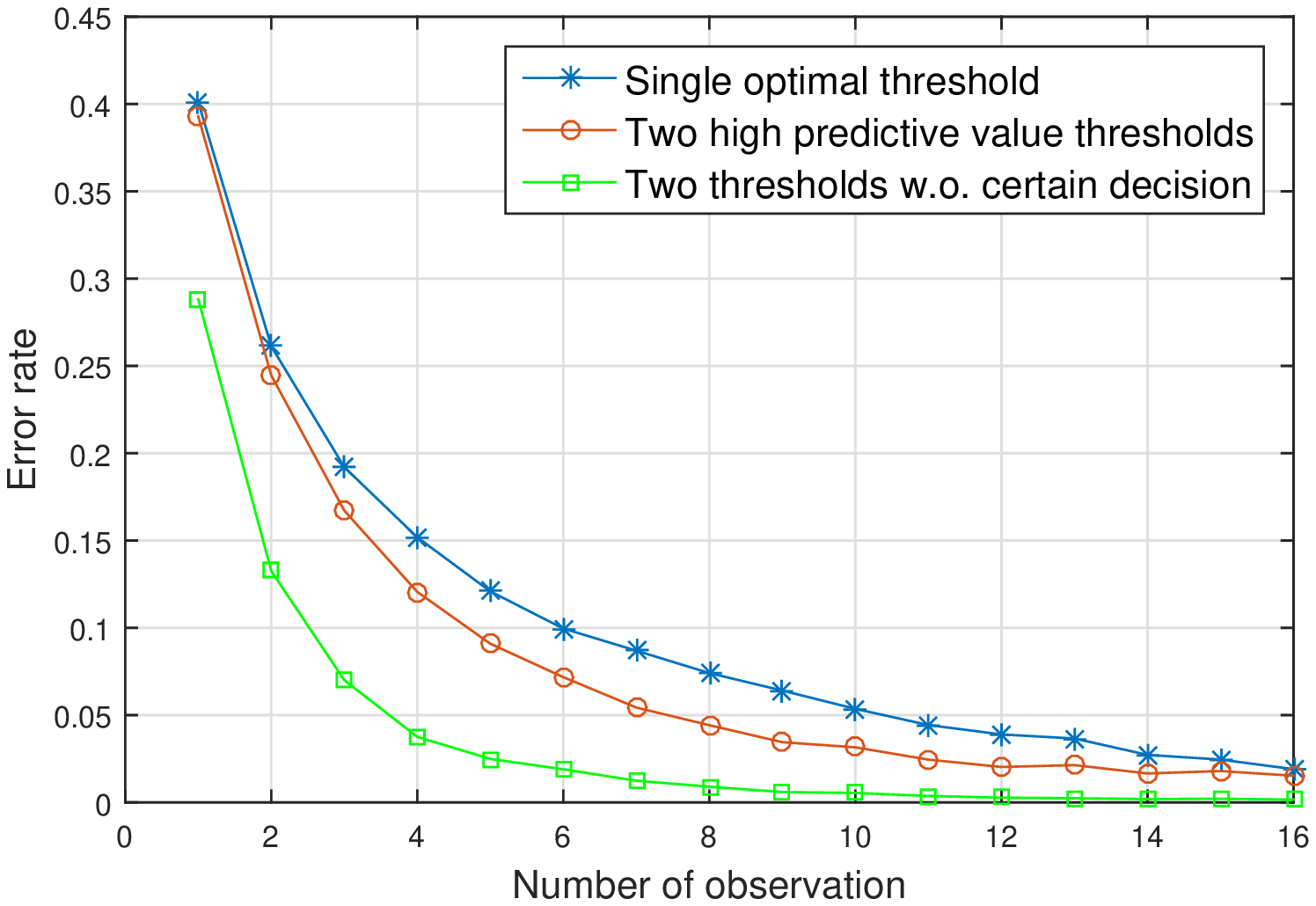}
 \label{fig:error}
}
&
\subfloat[]{
 \includegraphics[width=5.4cm, height=3.3cm]{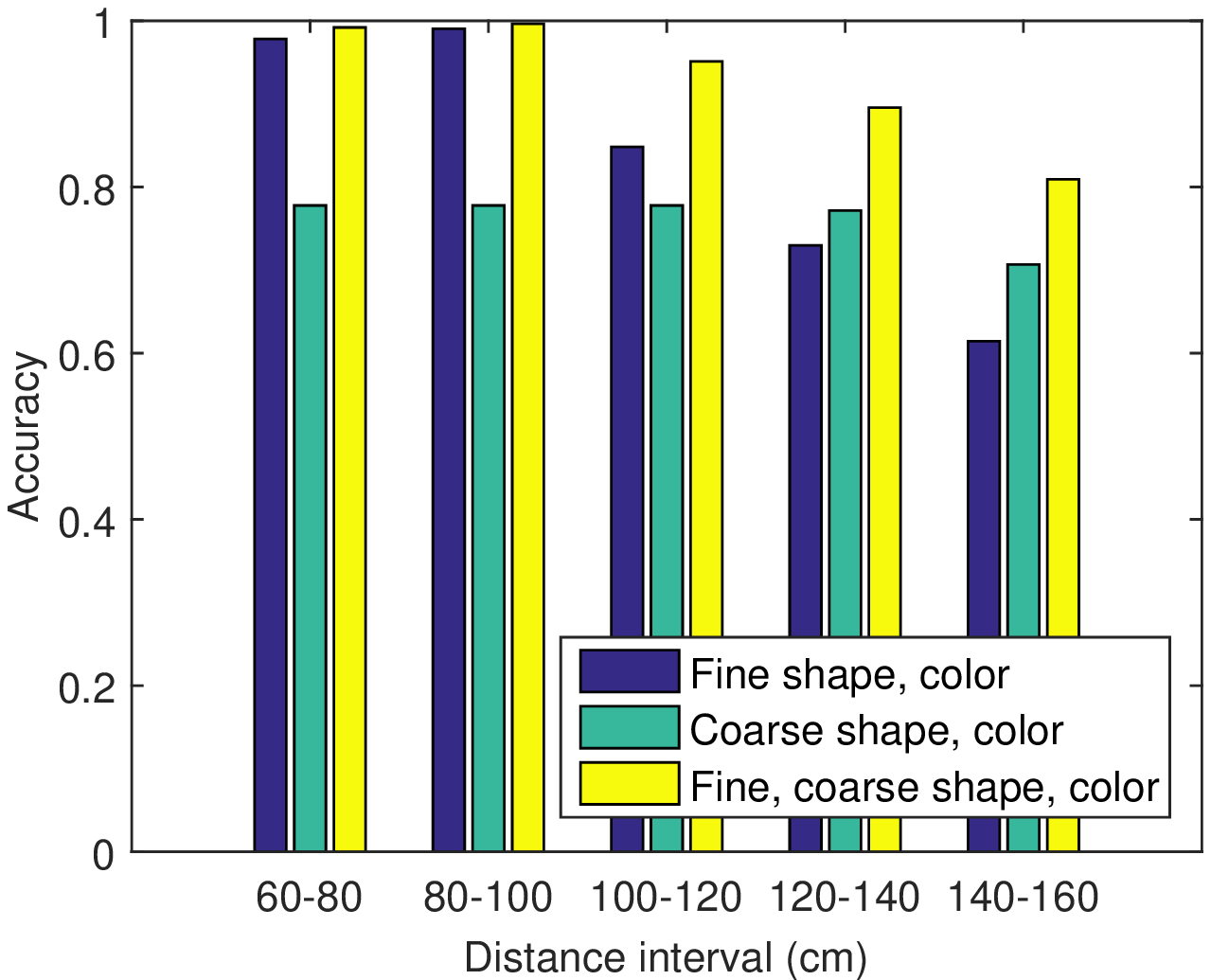}
 \label{fig:system}
}
\end{tabular}
\caption{(a): Error rate using classification with a single threshold (blue) and two high predictive value thresholds (red). The green line depicts the error component due to  the cases where  the two thresholds method has to randomly select. (b) Three systems' recognition accuracy for  different working distance intervals.}
\end{figure}

{\bf EXPERIMENT THREE: }The third experiment demonstrates the benefits of using less discriminative attributes for extending the system's working range. To recognize the $9$ objects
in Fig.~\ref{objects}, we build three recognition systems utilizing attributes of fine shape and  color, coarse shape and color, and all of the three attributes, respectively. Considering the influence of the  recognition distance on the response score distribution, the complete range of  distances from $60$ cm to $160$ cm is split into $5$ equal  intervals. We then learn the  classification thresholds and predictive values accordingly. Both, the training and the testing data, consist of around $100$ samples from each object across recognition distances from $60$ cm to $160$ cm. We learn the PPV and NPV by  counting the training data w.r.t. the thresholds and select thresholds satisfying a predictive value larger than $0.96$. 
The minimum detection rate for the reliable working distance interval is $0.09$. This means if 1) an attribute classifier cannot find a threshold with PPV larger than $0.96$, and 2)  detection rate  larger than $0.09$ in a certain distance interval, the output of this attribute classier in this interval will not be adopted for decision making.
During the testing phase, for fair comparison, we constrain the number of input point clouds collected from the same distance interval in each working region. Around $120$ point clouds are collected for each object. Once more, random selection is  applied when multiple objects are found  as possible candidates.

Fig.~\ref{fig:system} displays the systems' recognition accuracy after observing three times in each distance interval. As expected, the classification performance starts to decrease for  larger distances. At  $120$ cm to $160$ cm, the system using  fine shape attributes (blue)  performs even worse than the system using less selective coarse attributes (green). This validates that the coarse shape based classifier has a larger working region, though  its simple mechanisms allows for less discrimination than the fine grain attribute based classifier. Finally, due to the complementary properties, the system  using all  attributes (yellow) achieves the best performance at each working region.

\section{Conclusions}

In this work we put forward a practical framework for using multiple attributes for object recognition, which incorporates recognition distance into the decision making. Considering the difficulties of finding a single best classification threshold and the availability of multiple inputs at testing time, we propose to learn a high PPV and a high NPV threshold and discard  uncertain values  during decision making. The framework's correctness was  proven and a fundamental experiment was  conducted to demonstrate our approach's feasibility and benefits. Additionally, we showed that  less selective shape attributes (compared to the sophisticated ones) can have advantages,  because  their  simple mechanism can  lead to high reliability when the system is working at a large range of distances. 

In future work, we plan to extend the approach to a variety of environmental factors such as lighting conditions, blur, and occlusions. Furthemore, additional attribute classifiers will be incorporated to improve  the system's recognition performance.

{\bf Acknowledgment: }
This work was funded by the support of  DARPA (through ARO) grant W911NF1410384, by NSF through grants CNS-1544787 and SMA-1540917 and Samsung under the GRO program (N020477, 355022).

%
%

\bibliographystyle{splncs}
\bibliography{0670}
\end{document}